\documentclass{article}

\usepackage{microtype}
\usepackage{graphicx}
\usepackage{subfig}
\usepackage{adjustbox}
\usepackage{booktabs} 
\usepackage[utf8]{inputenc}
\usepackage{amsmath}
\usepackage{amsthm}
\usepackage{amssymb}
\usepackage{parskip}
\usepackage{mathtools}
\usepackage{mathrsfs}
\usepackage{adjustbox}
\usepackage{bbm}
\usepackage{color}
\usepackage{cases}
\usepackage[linewidth=1pt]{mdframed}

\usepackage{textcomp}
\def\BibTeX{{\rm B\kern-.05em{\sc i\kern-.025em b}\kern-.08em
    T\kern-.1667em\lower.7ex\hbox{E}\kern-.125emX}}

\allowdisplaybreaks

\makeatletter
\def\BState{\State\hskip-\ALG@thistlm}
\makeatother

\newtheorem{defn}{Definition}

\newtheorem*{thm*}{Theorem}
\newtheorem{thm}{Theorem}
\newtheorem{lemma}{Lemma}
\newtheorem*{lemma*}{Lemma}
\newtheorem{corollary}{Corollary}
\newtheorem*{corollary*}{Corollary}

\newtheorem*{remark*}{Remark}

\ifodd 1

\else

\fi

\ifodd 1
\newcommand{\com}[1]{{\color{red}(C: #1)}}
\else
\newcommand{\com}[1]{}
\fi

\ifodd 0
\newcommand{\rev}[1]{{\color{blue}#1}}
\else
\newcommand{\rev}[1]{#1}
\fi

\ifodd 0
\newcommand{\dtl}[1]{{\color{red}\\Details: #1}}
\else
\newcommand{\dtl}[1]{}
\fi

\ifodd 1
\newcommand{\alt}[1]{{\color{red}\\Alternative: #1}}
\else
\newcommand{\alt}[1]{}
\fi

\DeclareMathOperator*{\argmin}{arg\,min}
\DeclareMathOperator*{\argmax}{arg\,max}

\newcommand{\bs}{\boldsymbol}
\newcommand{\ssucc}{\overset{\ast}{\succ}}
\newcommand{\snsucc}{\overset{\ast}{\nsucc}}
\newcommand{\sprec}{\overset{\ast}{\prec}}
\newcommand{\snprec}{\overset{\ast}{\nprec}}
\renewcommand{\vec}{\boldsymbol}

\def\O{\mathcal{O}^\ast}
\def\Oh{\hat{\mathcal{O}}}
\def\P{\Pr}  
\def\E{\mathrm{E}}  

\usepackage[accepted]{icml2018}

\icmltitlerunning{Combinatorial Multi-Objective Multi-Armed Bandit Problem}

\begin{document}

\twocolumn[
\icmltitle{Combinatorial Multi-Objective Multi-Armed Bandit Problem}



\icmlsetsymbol{equal}{*}

\begin{icmlauthorlist}
\icmlauthor{Doruk Öner}{bilkent}
\icmlauthor{Altuğ Karakurt}{osu}
\icmlauthor{Atilla Eryılmaz}{osu}
\icmlauthor{Cem Tekin}{bilkent}
\end{icmlauthorlist}

\icmlaffiliation{bilkent}{Department of Electrical and Electronics Engineering, Bilkent University, Turkey}
\icmlaffiliation{osu}{Department of Electrical and Computer Engineering, The Ohio State University, OH, USA}

\icmlcorrespondingauthor{Doruk Öner}{doruk.oner@ug.bilkent.edu.tr}
\icmlcorrespondingauthor{Altuğ Karakurt}{karakurt.1@osu.edu}
\icmlcorrespondingauthor{Atilla Eryılmaz}{eryilmaz.2@osu.edu}
\icmlcorrespondingauthor{Cem Tekin}{cemtekin@ee.bilkent.edu.tr}

\icmlkeywords{ICML, Machine Learning, Multi-Armed Bandits, Telecommunications, Pareto Efficient, Multi Objective Optimization}

\vskip 0.3in
]



\printAffiliationsAndNotice{}  

\begin{abstract}
In this paper, we introduce the \textit{COmbinatorial Multi-Objective Multi-Armed Bandit} (COMO-MAB) problem that captures the challenges of combinatorial and multi-objective online learning simultaneously. In this setting, the goal of the learner is to choose an action at each time, whose reward vector is a linear combination of the reward vectors of the arms in the action, to learn the set of {\em super} Pareto optimal actions, which includes the Pareto optimal actions and actions that become Pareto optimal after adding an arbitrary small positive number to their expected reward vectors. We define the Pareto regret performance metric and propose a fair learning algorithm whose Pareto regret is $O(N L^3 \log T)$, where $T$ is the time horizon, $N$ is the number of arms and $L$ is the maximum number of arms in an action. 
We show that COMO-MAB has a wide range of applications, including recommending bundles of items to users and network routing, and focus on a resource-allocation application for multi-user communication in the presence of multidimensional performance metrics, where we show that our algorithm outperforms existing MAB algorithms.
\end{abstract}

\section{Introduction}
In the classical MAB problem \cite{lai_robbins} there is a set of stochastic processes, termed arms, with unknown statistics. At each time step, the learner selects an arm and obtains a random reward that depends on the selected arm. The goal of the learner is to maximize its long-term reward by using the previous observations to predict arm rewards. The main challenge in this problem is to balance exploration and exploitation. The learner should exploit to maximize its immediate reward, while it should explore to form better estimates of the arm rewards. It is shown in many works  that the optimal learning policy strikes the balance between these two \cite{lai_robbins,ucb,survey}. 
The \textit{combinatorial multi-armed bandit} (C-MAB) is proposed as an extension to the MAB problem \cite{gai}, where the learner chooses a set of the available arms, termed action, at each time step, and then, observes both the reward of the action (which is a linear combination of the rewards of the arms in the action) and the individual rewards of each arm in the action. This new dimension in the problem formulation causes the naive learner to suffer from the curse of dimensionality. As a remedy, numerous learning algorithms that exploit correlations between the arms to learn faster have been developed \cite{gai,kveton_linear,chen}.

Another line of work focuses on extending optimization problems with multiple performance criteria to an online learning setting. This extension is called \textit{multi-objective multi-armed bandit} (MO-MAB) problem \cite{drugan13,drugan14,auer_pareto}, where each arm yields multiple rewards when chosen. Hence, the reward of an arm is modeled as a random vector. In this setting, the ordering of arms become ambiguous due to the multi-dimensional aspect of the problem. For this reason, the learning objective in this problem is usually defined to be the Pareto front, which consists of arms that are incomparable with each other in terms of the reward. This extension also poses significant challenges compared to the classical MAB problem, due to the presence of multiple (and possibly conflicting) objectives. While C-MAB and MO-MAB has been studied separately, to the best of our knowledge, there exists no prior work that considers them jointly. 

This paper aims to solve the two online learning challenges presented above together, by introducing a combinatorial online learning problem with multidimensional performance metrics. As a solution concept, we develop a novel MAB model, which we name as Combinatorial Multi-objective MAB (COMO-MAB). Essentially, COMO-MAB is a fusion of Combinatorial MAB (C-MAB) and Multi-objective MAB (MO-MAB). In COMO-MAB, the learner selects an action that consists of multiple arms, and receives a reward vector, which consists of a linear combination of the reward vectors of the arms that are in the selected action. The learner also gets to observe the reward vectors of the arms that are in the action. 

We first show that learning in COMO-MAB is more challenging than learning in the classical MAB, since there might be actions with zero Pareto suboptimality gap that are not in the Pareto front. This motivates us to define the {\em super Pareto front} (SPF) that extends the Pareto front to include actions that can become Pareto optimal by adding an arbitrary small positive value to their expected rewards vectors. Then, we define the Pareto regret, which measures the cumulative loss of the learner due to not selecting actions that lie in the SPF. In order to minimize the regret of the learner, we propose an upper confidence bound (UCB) based algorithm (COMO-UCB), and prove that its regret by time $T$ is $O(N L^3 \log T)$, where $N$ denotes the number of arms and $L$ denotes the maximum number of arms in an action. COMO-UCB is fair in the sense that at each time step, it selects an action from the estimated SPF uniformly at random. As we show in the numerical results, this lets COMO-UCB to achieve reasonably high rewards in all objectives, instead of favoring one objective over the other objectives.
Our regret analysis requires a new set of technical methods that includes application of a multi-dimensional version of Hoeffding's inequality, defining a multi-dimensional notion of suboptimality, and evaluating the regret of the selected action by comparing it with the subset of the SPF, which dominates the selected action the most.

Later, we show how three important multi-objective learning problems in multi-user communication, recommender systems and network routing can be modeled using COMO-MAB.
Finally, we demonstrate the performance and fairness of COMO-UCB in a practical multi-user communication setup, and show that it outperforms other state-of-the-art MAB algorithms.

\section{Related Work}\label{sec:related}
Most MAB algorithms learn using index policies based on upper confidence bounds. These policies use the mean estimates of arms and an inflation term to compute an index for each arm, and then, choose the arm with the maximum index at each decision epoch. The motivation for this is to behave optimistically in the face of uncertainity using the inflation terms, which encourage the algorithm to explore the under-sampled arms instead of choosing the arm with the maximum mean estimate at all times. This approach was first used in \citet{lai_robbins} to design asymptotically optimal learning algorithms. Later, \citet{ucb} proposed the celebrated UCB1 algorithm, whose indices are very simple to compute, and showed that UCB1 achieves logarithmic regret uniformly over time. Many variants of UCB based index policies have been proposed since then (see \citet{survey} and references therein). In our proposed solution to COMO-MAB problem, we also use UCB indices.

Many existing works on C-MAB propose solutions inspired by UCBs to solve C-MAB problem, such as \citet{dani2008, abbasi, gai, kveton_linear} where the reward of an action is a linear combination of involved arms, as well as specialized versions of C-MAB problem like matroid bandits \cite{kveton_matroid}, cascading bandits \cite{kveton_cascading} and C-MAB with probabilistically triggered arms \cite{chen}. Some of these works allow the learner to choose a fixed number of arms at any time, while other works generalize this approach by defining an action set that the learner chooses from, where the reward of each action is modeled as a general (possibly non-linear) function of the arm rewards. Another influential work on C-MAB \cite{cesabianchi} proposes a randomized algorithm inspired by GeometricHedge algorithm in \citet{dani2008price}.  
 
MO-MAB problem is studied through numerous different approaches. In \citet{bubeck}, each objective is considered as a different MAB problem and the aim is to find the optimal arm for each objective separately. Another line of work is interested in identifying the Pareto front of the arms \cite{drugan14, auer_pareto}, while another work aims to generalize the notion of regret in single-objective bandits to multi-objective bandits by defining the Pareto regret \cite{drugan13}. The former approach tries to maximize the probability of choosing a Pareto optimal arm, while the latter approach intends to minimize the Pareto regret. There also exists other variants of the MO-MAB problem such as the contextual MO-MAB \cite{tekin2017multi} and $\chi$-armed MO-MAB \cite{van2014multi}.
\section{Problem Description} \label{sec:problem}
In COMO-MAB, there exists $N$ arms indexed by the set ${\cal N} := [N]$, where $[N]$ denotes the set of positive integers from $1$ to $N$. The $D$-dimensional random reward vector of arm $i$ at time step $t$, denoted by $\bs{X}_i(t) := [X^{(1)}_i(t),\ldots,X^{(D)}_i(t)]$, is drawn from an unknown distribution with finite support, which is assumed to be the unit hypercube $[0,1]^D$ without loss of generality, independent of other time steps.\footnote{Independence is only required over time steps and not over different objectives.} Here, $X^{(j)}_i(t)$ denotes the random reward of arm $i$ in objective $j$, where the objectives are indexed by the set ${\cal D} := [D]$. The mean vector of arm $i$ is denoted by $\vec{\mu}_i := [ \mu^{(1)}_i, \ldots, \mu^{(D)}_i ]$, where $\mu^{(j)}_i := \mathbb{E} [X^{(j)}_i(t)]$.

We use ${\cal A}$ to denote the finite set of actions, where each action $\vec{a}$ is represented as an $N$-dimensional real-valued vector, i.e., $\vec{a} := (a_1, \ldots, a_N)$.
%
Moreover, for an action $\vec{a}$, we assume that $a_i \geq 0$ for all $i \in {\cal N}$. We say that arm $i$ is in action $\vec{a}$, if $a_i >0$. The set of arms in action $\vec{a}$ is given as ${\cal N}_{\text{nz}}(\vec{a}(t)) := \{ i \in {\cal N}: a_i(t) \neq 0 \}$, and the maximum number of arms in an action is given as $\textit{L} = \max_{\bs{a} \in {\cal A}}| {\cal N}_{\text{nz}}(\bs{a})|$.
The $D$-dimensional reward vector of action $\vec{a}$ in time step $t$ is given by $\vec{R}_{\vec{a}}(t) = \sum_{i=1}^n a_i \vec{X}_i(t)$,\footnote{In COMO-MAB, the same scalar $a_i$ multiplies the rewards of arm $i$ in all of the $D$ objectives when action $\vec{a}$ is selected. COMO-MAB can be generalized such that the multiplier for different objectives of the same arm becomes different. This can be achieved by defining $a_i$ as a $D$-dimensional vector. Our results can be extended to this case in a straightforward manner.} and its mean reward vector is given by $\vec{\mu}_{\vec{a}} = \sum_{i=1}^n a_i \vec{\mu}_i$. 
COMO-MAB can be used to model many applications that involve combinatorial action sets and multi-dimensional performance metrics such as multi-user communication, recommender systems and network routing (see Section \ref{sec:app} for a detailed discussion).

Since the action rewards are multi-dimensional in COMO-MAB, in order to compare different actions, one can think of using the notion of Pareto optimality. Identifying the set of arms in the Pareto front by using the sample mean estimates of the rewards can be challenging, since there might be actions not in the Pareto front for which the suboptimality gap is zero. This motivates us to define the notion of {\em super Pareto optimality} (SPO), which extends Pareto optimality in order to account for such actions.

\begin{defn}[SPO]
(i) An action $\vec{a}$ is \emph{weakly dominated} by action $\vec{a}'$, denoted by $\vec{\mu}_{\vec{a}} \preceq \vec{\mu}_{\vec{a}'}$ or $\vec{\mu}_{\vec{a}'} \succeq \vec{\mu}_{\vec{a}}$, if $\mu_{\vec{a}}^{(j)} \leq \mu_{\vec{a}'}^{(j)}, \forall j \in {\cal D}$. \\
(ii) An action $\vec{a}$ is \emph{dominated} by action $\vec{a}'$, denoted by $\vec{\mu}_{\vec{a}} \prec \vec{\mu}_{\vec{a}'}$ or $\vec{\mu}_{\vec{a}'} \succ \vec{\mu}_{\vec{a}}$, if it is weakly dominated and $\exists j \in {\cal D}$ such that $\mu_{\vec{a}}^{(j)} < \mu_{\vec{a}'}^{(j)}$. \\
(iii) An action $\vec{a}$ is \emph{super-dominated} by action $\vec{a}'$, denoted by $\vec{\mu}_{\vec{a}} \sprec \vec{\mu}_{\vec{a}'}$ or $\vec{\mu}_{\vec{a}'} \ssucc \vec{\mu}_{\vec{a}}$, if $\mu_{\vec{a}}^{(j)} < \mu_{\vec{a}'}^{(j)}, \forall j \in {\cal D}$. \\
(iv) Two actions $\vec{a}$ and $\vec{a}'$ are incomparable, denoted by $\vec{\mu}_{\vec{a}} || \vec{\mu}_{\vec{a}}$, if neither action super-dominates the other. \\
(v) An action is SPO if it is not super-dominated by any other action. The set of all SPO actions is called the SPF, and is denoted by $\O$.
\end{defn}

Note that SPO is a relaxed version of Pareto optimality. Every Pareto optimal action is also SPO. Moreover, an action that is not Pareto optimal can be SPO if adding any $\epsilon>0$ to the mean value of any dimension of the reward of that action makes it a Pareto optimal action.
\begin{remark*}
	The SPF $\neq$ the Pareto front happens in very specific problems that involve some kind of symmetry. For instance, given three actions with expected rewards $(2,1)$, $(1,2)$ and $(1,1)$, the Pareto front contains the first two actions, while the SPF contains all actions. However, such symmetric cases rarely exist in combinatorial problems of our interest.\footnote{We have not encountered such a case in our simulations.}
\end{remark*}

At time step $t$ the learner selects an action $\vec{a}(t) \in {\cal A}$, and receives the reward vector $\vec{R}_{\vec{a}(t)}(t)$. Then, at the end of time step $t$, it observes the reward vectors of the arms in ${\cal N}_{\text{nz}}(\vec{a}(t))$.
We measure the performance of the learner using the notion of {\em Pareto regret}, which is a generalization of the Pareto regret definition for the $K$-armed bandit problem \cite{drugan13} to our combinatorial setting. For this, we first define the {\em Pareto suboptimality gap} (PSG) of an action, which measures the distance between an action and the Pareto front. 

\begin{defn}[PSG]
	The PSG of an action $\vec{a} \in {\cal A}$, denoted by $\Delta_{\vec{a}}$, is defined as the minimum scalar $\epsilon \geq 0$ that needs to be added to all entries of $\vec{\mu}_{\vec{a}}$ such that $\vec{a}$ becomes a member of the SPF. More formally, 
	\begin{align*}
	\Delta_{\vec{a}} := \min_{\epsilon \geq 0} \epsilon ~~ \text{such that} ~~  (\vec{\mu}_{\bs{a}} + \bs{\epsilon}) \> || \> \vec{\mu}_{\vec{a}'}, \forall \vec{a}' \in \O 
	\end{align*}
	where $\bs{\epsilon}$ is a $D$-dimensional vector, whose all entries are $\epsilon$. In addition, we define the extrema of the PSG as follows:
$\Delta_{\max} := \max_{\vec{a} \in \mathcal{A}} \Delta_{\vec{a}}$ and 
$\Delta_{\min} := \min_{\vec{a} \in \mathcal{A} - {\cal O}^*} \Delta_{\vec{a}}$.
\end{defn}

Based on the definition of PSG, the {\em Pareto regret} (simply referred to as the {\em regret} hereafter) of the learner by time step $T$ is given as
\begin{align*}
\text{Reg}(T) := \sum_{t=1}^{T} \Delta_{\vec{a}(t)}.
\end{align*}
Note that $\Delta_{\vec{a}} > 0, \forall \vec{a} \notin \O$ and $\Delta_{\vec{a}} = 0, \forall \vec{a} \in \O$. Hence, the actions that are in the SPF but not in the Pareto front also have zero PSG, and their selection does not contribute to the regret.
In the following section, we propose a learning algorithm that minimizes $\mathbb{E} [\text{Reg}(T)]$, which also ensures that each action in the estimated SPF is selected with an equality probability.

\section{The Combinatorial Multi-Objective Upper Confidence Bound Algorithm}\label{sec:COM-UCB}

In this section we propose {\em COmbinatorial Multi-objective Upper Confidence Bound} (COMO-UCB) algorithm whose pseudocode is given in Algorithm 1.
\setlength{\textfloatsep}{3em}
\begin{algorithm}[h!] 
   \caption{COMO-UCB}
    \begin{algorithmic}
        \STATE // INITIALIZATION
        \STATE $\textit{L}$, $\hat{\bs{\mu}}_i = \bs{0}$, $m_i=0$, $\forall i \in {\cal N}$
        \FOR{$i = 1$ to $N$}
        	\STATE $t = i$
            \STATE Select $\bs{a}$ uniformly at random from $\{ \bs{a} \in {\cal A}: a_i \neq 0 \}$
            \STATE Collect reward $\vec{R}_{\vec{a}}(t)$
            \STATE Observe reward vectors $\bs{X}_i(t)$, $\forall i \in {\cal N}_{\text{nz}}(\bs{a})$
            \STATE $\hat{\bs{\mu}}_i = (\hat{\vec{\mu}}_i m_i + \bs{X}_i(t))/(m_i+1)$, $\forall i \in {\cal N}_{\text{nz}}(\bs{a})$
            \STATE $m_i = m_i + 1$, $\forall i \in {\cal N}_{\text{nz}}(\vec{a})$
        \ENDFOR
        \STATE // MAIN LOOP
        \WHILE{$1$}
        	\STATE $t = t + 1$
            \STATE Find the estimated SPF $\Oh$:
\begin{align}
\Oh = &\left\{ \vec{a} \in {\cal A} : \sum_{i \in {\cal N}} a_i \bigg(\hat{\vec{\mu}}_i + C_i(t) \bigg) \right. \notag \\ & \hspace{0.3in}  \left. \snprec \sum_{i \in {\cal N}} a'_i\bigg(\hat{\bs{\mu}}_i + C_i(t) \bigg), \forall \vec{a}' \in {\cal A} \right\}. \label{eqn:UCB}
\end{align}
            \STATE Select $\bs{a}$ uniformly at random from $\Oh$
            \STATE Collect reward $\vec{R}_{\vec{a}}(t)$
		    \STATE Observe reward vectors $\bs{X}_i(t)$, $\forall i \in {\cal N}_{\text{nz}}(\bs{a})$
			\STATE $\hat{\vec{\mu}}_i = (\hat{\vec{\mu}}_i m_i + \bs{X}_i(t))/(m_i+1)$, $\forall i \in {\cal N}_{\text{nz}}(\bs{a})$
			\STATE $m_i = m_i + 1$, $\forall i \in {\cal N}_{\text{nz}}(\bs{a})$
        \ENDWHILE
\end{algorithmic}
\end{algorithm}
For each arm $i \in {\cal N}$, COMO-UCB keeps two parameters that are updated at each time step: $\hat{\bs{\mu}}_i$ and $m_i$. The first one is the sample mean reward vector of arm $i$ and the second one is the number of time steps in which arm $i$ is selected. When the learner chooses an action $\bs{a}$ such that $i \in {\cal N}_{\text{nz}}(\bs{a})$, then $m_i$ is incremented by $1$. When explicitly referring to the value of these counters at the end of time step $t$ we use $\hat{\bs{\mu}}_i(t)$ and $m_i(t)$.\footnote{We adopt this convention for other variables that change over time as well.}

In the first $N$ time steps, the learner selects actions such that each arm gets selected at least once. This is done to ensure proper initialization of $\hat{\bs{\mu}}_i$, $i \in {\cal N}$. After this, the learner computes the estimated SPF, $\Oh$, at the beginning of each time step by using UCBs for the reward vectors of the arms as given in \eqref{eqn:UCB}, where $C_{i}(t) = \sqrt{\frac{(L+1)\log{((t-1)\sqrt[4]{D})}}{m_{i}(t)}}$, is the inflation term which serves as a proxy for the learner's uncertainty about the expected reward of arm $i$. This allows the learner to explore rarely selected arms, since it forces the actions that put large weights to the rarely selected arms to be in $\Oh$. In addition, the randomization in action selection ensures that the learner does not favor any action in $\Oh$. As we show in Section \ref{sec:app}, the randomization feature of COMO-UCB, which is not necessary for the classical MAB algorithms to minimize their regret, allows it to collect reasonably high rewards in all objectives, without favoring any of the objectives over others.
After an action is selected, the learner observes the reward vectors of the arms that have non-zero weights in the selected action, and updates the sample mean reward vector of these arms.
\section{Regret Analysis} \label{sec:regret}
In this section we bound the expected regret of COMO-UCB. The main result of this section is given in the following theorem (a more detailed version of the proof is given in the supplemental document). 
\begin{thm}\label{thm:mainregret}
When run with $C_{i}(t) = \sqrt{\frac{(L+1)\log{((t-1)\sqrt[4]{D})}}{m_{i}(t)}}$ the expected regret of COMO-UCB is bounded by
\begin{align*}
\mathbb{E} [\mathrm{Reg}(T)] &\leq \Delta_{\max}\Big( \frac{4a^2_{\max}NL^2(L+1)\log(T\sqrt[4]{D})}{\Delta_{\min}^2} \\ &\hspace{11em}+ N + \frac{\pi^2}{3}NL \Big)
\end{align*}
where $a_{\max} = \max_{\vec{a} \in \mathcal{A}} \{ \max_q a_q \}$.
\end{thm}

\textit{Proof.} We first state a version of Hoeffding's inequality adapted to multiple dimensions.

\begin{lemma}\label{lemma:hoeff} \cite{drugan13}
Let $\vec{\mu}$ be the mean vector of a $D$-dimensional i.i.d. process with support $[0,1]^D$ and $\hat{\vec{\mu}}_n$ denote the sample mean estimate of $\vec{\mu}$ based on $n$ observations. Then, for any $k \in \mathbb{R}_+$
$\Pr (\vec{\mu} + k \snsucc \hat{\vec{\mu}}_n) \leq De^{-2nk^2}$ and 
$\Pr (\vec{\mu} - k \snprec \hat{\vec{\mu}}_n) \leq De^{-2nk^2}$.
\end{lemma}
\dtl{ 
	\begin{proof}
\begin{align*}
\Pr (\vec{\mu} + k \snsucc \hat{\vec{\mu}}_n) 
&= \Pr \Bigg( \bigcup_{j = 1}^D 
\left\{ \mu^{(j)} + k \leq \hat{\mu}_n^{(j)} \right\} \Bigg) \\
&\leq \sum_{i=1}^D \Pr (\mu^{(j)} + k \leq \hat{\mu}_n^{(j)}) \\ 
&\leq \sum_{i=1}^D e^{-2nk^2} = De^{-2nk^2}
\end{align*}
where we use Hoeffding's inequality to obtain the last inequality.
Similarly, the reverse direction of Hoeffding's inequality is the following.
\begin{align}\label{ch2}
\Pr (\mu^{(j)} - k \geq \hat{\mu}_n^{(j)}) \leq e^{-2nk^2}.
\end{align}
As before, we decompose Pareto dominance relationship into a union of events and use (\ref{ch2}) with union bound.
\begin{align*}
\Pr (\vec{\mu} - k \snprec \hat{\vec{\mu}}_n) 
&= \Pr \Bigg( \bigcup_{j = 1}^D 
\left\{ \mu^{(j)} - k \geq \hat{\mu}_n^{(j)} \right\} \Bigg) \\
&\leq \sum_{i=1}^D \Pr (\mu^{(j)} - k \geq \hat{\mu}_n^{(j)}) \\ 
&\leq \sum_{i=1}^D e^{-2nk^2} = De^{-2nk^2}.
\end{align*}
\end{proof}
}

Let $T_{\bs{a}}(t)$ denote the number of times action $\bs{a}$ is selected in the first $t$ time steps. Next, we define a set of auxiliary counters, denoted by $\tilde{\bs{T}}(t) := \{ \tilde{T}_i(t) \}_{i \in {\cal N}}$, that will be used in the regret analysis. We borrow the idea of using such auxiliary counters from \citet{gai}, which used these counters for analyzing the regret of a single-objective combinatorial bandit problem. However, our analysis is significantly different from the analysis in \citet{gai} due to the fact that the definition of suboptimality is very different for the multi-objective setting.

Let $\tilde{\bs{T}}(t)$ be defined for $t > N$ such that if an action in $\O$ is selected in time step $t$,\footnote{$\tilde{\bs{T}}(t)$ is equal to the zero vector for $t \leq N$.} then $\tilde{\bs{T}}(t) = \tilde{\bs{T}}(t-1)$, while if an action not in  $\O$ is selected in time step $t$, then $\tilde{T}_{i^\ast(t)}(t) = \tilde{T}_{i^\ast(t)}(t-1)+1$ for $i^\ast(t) = \argmin_{i \in {\cal N}_{\text{nz}}(\bs{a}(t))} m_i(t)$ and $\tilde{T}_{i}(t) = \tilde{T}_{i}(t-1)$ for $i \neq i^\ast(t)$.\footnote{In case $\argmin_{i \in {\cal N}_{\text{nz}}(\bs{a}(t))} m_i(t)$ contains multiple elements, an arbitrary element is selected to be $i^\ast(t)$.}
Since exactly one element of $\tilde{\bs{T}}(t)$ is incremented by $1$ in every time step in which a suboptimal action is selected, we have
$\sum_{\bs{a} : \Delta_{\bs{a}} >0 } T_{\bs{a}}(t) = \sum_{i \in {\cal N}} \tilde{T}_{i}(t)$. Thus,
\begin{align}\label{t_change}
\sum_{\bs{a} \notin \O} \mathbb{E} [T_{\bs{a}}(t)] = \sum_{i = 1}^N \mathbb{E} [\tilde{T}_{i}(t)].
\end{align}
In addition, we have by definition
$\tilde{T}_{i}(t) \leq m_i(t) \text{, } \forall i \in {\cal N}$.
We also have
\begin{align}
&\tilde{T}_i(T) 
= \sum_{t = N+1}^T \mathbbm{1}\{ i^*(t) = i, \bs{a}(t) \notin \O \} \notag \\
&\leq l + \sum_{t = N}^{T-1} \mathbbm{1}\{   l \leq m_h(t), \forall h \in {\cal N}_{\text{nz}}(\bs{a}(t+1)), \notag \\ 
&\hspace{14em}~\vec{a}(t+1) \not \in \O   \} 
\label{eqn:a1}
\end{align}
\dtl{Note that 
\begin{align*}
& \Bigg\{  i^*(t+1) = i,~ \tilde{T}_i(t) \geq l,~ \vec{a}(t+1) \not \in \O \Bigg\} \\
& \subseteq \Bigg\{  \tilde{T}_i(t+1) \geq l+1,~ i^\ast(t+1) = i, ~\vec{a}(t+1) \not \in \O  \Bigg\} \\
& \subseteq \Bigg\{  l+1 \leq m_h(t+1), \forall h \in {\cal N}_{\text{nz}}(\bs{a}(t+1)), ~\vec{a}(t+1) \not \in \O  \Bigg\} \\
& \subseteq \Bigg\{  l \leq m_h(t), \forall h \in {\cal N}_{\text{nz}}(\bs{a}(t+1)), ~\vec{a}(t+1) \not \in \O \Bigg\}.
\end{align*}
The first equality holds because when $i^\ast(t+1)=i$ and $\bs{a}(t+1) \notin \O$, we have $\tilde{T}_i(t+1)=\tilde{T}_i(t)+1$.}
where $\mathbbm{1}(\cdot)$ denotes the indicator function.
Let $C_{t,m} := \sqrt\frac{(L+1)\log(t \sqrt[4]{D})}{m}$ and let
$\bar{\vec{\mu}}_{i,m}$ be the random vector that denotes the sample mean vector of $m$ reward vector observations from arm $i$. 

Let $\O_{\vec{a}(t)}$ denote the subset of SPF that super dominates $\vec{a}(t)$. If $\vec{a}(t) \not \in \O$, then this set is non-empty. Next, we define $\vec{a}'(t)$ as the action that dominates $\vec{a}(t)$ the most, which is given as
\begin{numcases}{\vec{a}'(t)=\hspace{-0.5em}}
  \argmax_{\vec{a}^\ast \in \O_{\vec{a}(t)}} 
  \left\{ \min_{1 \leq j \leq D} (\mu_{\vec{a^\ast}}^{(j)} - \mu_{\vec{a}(t)}^{(j)}) \right\} \text{ }& \hspace{-2.5em}if $\vec{a}(t) \not\in \O$\nonumber \\
  \hspace{11.5em}\vec{a}(t), &\hspace{-2.5em} otherwise. \nonumber
\end{numcases}

Next, we continue upper bounding \eqref{eqn:a1}:
\begin{align}
&\tilde{T}_i(T) \notag \\
&\leq l + \sum_{t = N}^{T-1} \mathbbm{1} 
\Bigg\{ \sum_{h \in {\cal N}_{\text{nz}}(\vec{a}(t+1))} \hspace{-1em}a_{h}(t+1)
 (\bar{\vec{\mu}}_{h, m_h(t)} + C_{t, m_h(t)}) \notag \\ &\hspace{4em} \snprec 
 \sum_{k \in {\cal N}_{\text{nz}}(\vec{a}'(t+1))}\hspace{-1em}  a'_k(t+1) (\bar{\vec{\mu}}_{k, m_k(t)} + C_{t, m_k(t)}), \notag \\ &\hspace{7em}\text{ }
 l \leq m_h(t), \forall h \in {\cal N}_{\text{nz}}(\bs{a}(t+1)) \Bigg\} . \label{eqn:proof2}
\end{align}

The following is an upper bound on (\ref{eqn:proof2}):
\begin{align}\label{last_bound}
\begin{split}
\tilde{T}_i(T) \leq &l + \sum_{t=N}^{T-1} \sum_{c_{h_1}=l}^{t}\sum_{c_{h_2}=l}^{t} \cdots \sum_{c_{h_{|{\cal N}_{\text{nz}}(\vec{a}(t+1))|}}=l}^{t} \\ &\sum_{d_{k_1}=1}^{t} \sum_{d_{k_2}=1}^{t} \cdots\sum_{d_{k_{|{\cal N}_{\text{nz}}(\vec{a'}(t+1))|}}=1}^{t}\\&\mathbbm{1}\Bigg\{ \sum_{q = 1}^{|{\cal N}_{\text{nz}}(\vec{a}(t+1))|} \hspace{-1em}a_{h_q}(t+1)(\bar{\vec{\mu}}_{h_q, c_{h_q}} + C_{t, c_{h_q}})\\&\snprec\hspace{-1em} \sum_{q = 1}^{|{\cal N}_{\text{nz}}(\vec{a'}(t+1))|}\hspace{-1em} a'_{k_q}(t+1)(\bar{\vec{\mu}}_{k_q, d_{k_q}} + C_{t, d_{k_q}}) \Bigg\}
\end{split}
\end{align}
where $h_q$ denotes the $q$th element of ${\cal N}_{\text{nz}}(\vec{a}(t+1))$ and $k_q$ denotes the $q$th element of ${\cal N}_{\text{nz}}(\vec{a}'(t+1))$.
The event inside the indicator function in (\ref{last_bound}) is true only if at least one of the following events occur:

\begin{align}\label{event1}
\begin{split}
&\mathcal{E}_1 := \Bigg\{ \sum_{q = 1}^{|\mathcal{N}_{nz}(\vec{a}'(t+1))|}\hspace{-1em} a'_{k_q}(t+1)\bar{\vec{\mu}}_{k_q, d_{k_q}} \snsucc \vec{\mu}_{\vec{a}'(t+1)} \\ &\hspace{3em}- \sum_{q = 1}^{|\mathcal{N}_{nz}(\vec{a}'(t+1))|}\hspace{-1em} a'_{k_q}(t+1)C_{t,d_{k_q}} \Bigg\}
\end{split}
\end{align}
\begin{align}\label{event2}
\begin{split}
&\mathcal{E}_2 := \Bigg\{ \vec{\mu}_{\vec{a}(t+1)} + \sum_{q = 1}^{|\mathcal{N}_{nz}(\vec{a}(t+1))|}\hspace{-1em} a_{h_q}(t+1) C_{t, c_{h_q}} \\ &\hspace{3em}\snsucc \sum_{q = 1}^{|\mathcal{N}_{nz}(\vec{a}(t+1))|}\hspace{-1em} a_{h_q}(t+1) \bar{\vec{\mu}}_{h_q, c_{h_q}} \Bigg\}
\end{split}
\end{align}
\begin{align}\label{event3}
&\mathcal{E}_3 := \Bigg\{\vec{\mu}_{\vec{a}'(t+1)} \nsucceq \vec{\mu}_{\vec{a}(t+1)} + 2\hspace{-1.5em}\sum_{q = 1}^{|\mathcal{N}_{nz}(\vec{a}(t+1))|}\hspace{-1.5em} a_{h_q}(t+1) C_{t, c_{h_q}}\Bigg\}.
\end{align}
\dtl{
Suppose that none of the events in (\ref{event1}), (\ref{event2}) and (\ref{event3}) are true, which implies that
\begin{align}\label{h1}
\begin{split}
&\sum_{q = 1}^{|\mathcal{N}_{nz}(\vec{a}'(t+1))|} \hspace{-1em}a'_{k_q}(t+1)\bar{\vec{\mu}}_{k_q, d_{k_q}} + \sum_{q = 1}^{|\mathcal{N}_{nz}(\vec{a}'(t+1))|} \hspace{-1em}a'_{k_q}(t+1)C_{t,d_{k_q}} \\ &\hspace{12em}\ssucc \vec{\mu}_{\vec{a}'(t+1)}
\end{split}
\end{align}
\begin{align}\label{h2}
\begin{split}
&\vec{\mu}_{\vec{a}(t+1)} + \sum_{q = 1}^{|\mathcal{N}_{nz}(\vec{a}(t+1))|} \hspace{-1em}a_{h_q}(t+1) C_{t, c_{h_q}} \\ &\hspace{3em}\ssucc \sum_{q = 1}^{|\mathcal{N}_{nz}(\vec{a}(t+1))|} \hspace{-1em}a_{h_q}(t+1) \bar{\vec{\mu}}_{h_q, c_{h_q}}
\end{split}
\end{align}
\begin{align}\label{h3}
\begin{split}
\vec{\mu}_{\vec{a}'(t+1)} \succeq \vec{\mu}_{\vec{a}(t+1)} + 2\sum_{q = 1}^{|\mathcal{N}_{nz}(\vec{a}(t+1))|} \hspace{-1em}a_{h_q}(t+1) C_{t, c_{h_q}}.
\end{split}
\end{align}
Now adding the term $\sum_{q = 1}^{|\mathcal{N}_{nz}(\vec{a}(t+1))|} a_{h_q}(t+1) C_{t, c_{h_q}}$ to both sides of (\ref{h2}) and combining the expressions in (\ref{h2}) and (\ref{h3}) together, we obtain the following:
\begin{align*}
&\vec{\mu}_{\vec{a}'(t+1)} \succeq \vec{\mu}_{\vec{a}(t+1)} + 2\hspace{-1em}\sum_{q = 1}^{|\mathcal{N}_{nz}(\vec{a}(t+1))|}\hspace{-1em} a_{h_q}(t+1) C_{t, c_{h_q}} \\ &\hspace{3.5em}\ssucc \sum_{q = 1}^{|\mathcal{N}_{nz}(\vec{a}(t+1))|}\hspace{-1em} a_{h_q}(t+1)\bar{\vec{\mu}}_{h_q, c_{h_q}} \\ &\hspace{8em}+\sum_{q = 1}^{|\mathcal{N}_{nz}(\vec{a}(t+1))|}\hspace{-1em} a_{h_q}(t+1) C_{t, c_{h_q}}.
\end{align*}

This equation together with (\ref{h1}) yields the following, which makes the term in the indicator function of (\ref{last_bound}) false.

\begin{align*}
&\sum_{q = 1}^{|\mathcal{N}_{nz}(\vec{a}'(t+1))|}\hspace{-1em} a'_{k_q}(t+1)(\bar{\vec{\mu}}_{k_q, d_{k_q}} + C_{t,d_{k_q}}) \\ &\hspace{5em}\ssucc \sum_{q = 1}^{|\mathcal{N}_{nz}(\vec{a}(t+1))|}\hspace{-1em} a_{h_q}(t+1)(\bar{\vec{\mu}}_{h_q, c_{h_q}}+C_{t, c_{h_q}}).
\end{align*}
}

Next, we continue by bounding the probabilities of the events given in
(\ref{event1}), (\ref{event2}) and (\ref{event3}).  
For (\ref{event1}), we have
\begin{align*} 
\P(\mathcal{E}_1) \leq &\P \Bigg(\bigcup_{q = 1}^{|\mathcal{N}_{nz}(\vec{a}'(t+1))|}\hspace{-1em} \left\{a'_{k_q}(t+1) \bar{\vec{\mu}}_{k_q, d_{k_q}} \right.
\\& \left. \hspace{5em}\snsucc a'_{k_q}(t+1)(\vec{\mu}_{k_q} - C_{t, d_{k_q}}) \right\} \Bigg) \\
&\leq \sum_{q=1}^{|\mathcal{N}_{nz}(\vec{a}'(t+1))|} \hspace{-1em}\P\Big( \bar{\vec{\mu}}_{k_q, d_{k_q}} \snsucc \vec{\mu}_{k_q} - C_{t,d_{k_q}}\Big).
\end{align*}
\dtl{
Assume that $A = \sum_{q} w_q A_q$ and $B = \sum_{q} w_q B_q$ where $w_q \geq 0$. Then, $A \snsucc B$ implies that $\exists q: A_q \snsucc B_q$. This holds because if $A_q \ssucc B_q$ for all $q$, then we must have $A \ssucc B$.}

Using the multi-dimensional Hoeffding's inequality (Lemma \ref{lemma:hoeff}), we obtain
\begin{align*}
&\P\Big( \bar{\vec{\mu}}_{k_q, d_{k_q}} \snsucc \vec{\mu}_{k_q} - C_{t,d_{k_q}} \Big)\leq De^{-2C^2_{t, d_{k_q}}d_{k_q}} \\ &\hspace{3em}=De^{-2(L+1)\log(t\sqrt[4]{D})} \leq t^{-2(L+1)}.
\end{align*}
Hence, the sum of $|\mathcal{N}_{nz}(\vec{a}'(t+1))|$ such probabilities yield:
\begin{align*}
&\P(\mathcal{E}_1) \leq |\mathcal{N}_{nz}(\vec{a}'(t+1))|t^{-2(L+1)} \leq Lt^{-2(L+1)}.
\end{align*}
\dtl{Similarly, for (\ref{event2}), we have
\begin{align*}
&\P\Bigg( \vec{\mu}_{\vec{a}(t+1)} + \sum_{q = 1}^{|\mathcal{N}_{nz}(\vec{a}(t+1))|} \hspace{-1em}a_{h_q}(t+1) C_{t, c_{h_q}} \\ &\hspace{9em}\snsucc \sum_{q = 1}^{|\mathcal{N}_{nz}(\vec{a}(t+1))|} \hspace{-1em}a_{h_q}(t+1) \bar{\vec{\mu}}_{h_q, c_{h_q}} \Bigg) \\
&\leq \sum_{q=1}^{|\mathcal{N}_{nz}(\vec{a}(t+1))|} \hspace{-1em}\P\left( \rev{\vec{\mu}_{h_q}} + C_{t,c_{h_q}} \snsucc \bar{\vec{\mu}}_{h_q, \rev{c_{h_q}}} \right).
\end{align*}}
Similarly, for (\ref{event2}), we have $\P(\mathcal{E}_2) \leq Lt^{-2(L+1)}$.

Finally, we bound the probability of (\ref{event3}).
Observe that for $l \geq \bigg\lceil \frac{4a^2_{max}L^2(L+1)\log(T\sqrt[4]{D})}{\Delta_{\vec{a}(t+1)}^2} \bigg\rceil$, we have the following:
\begin{align*}
&\vec{\mu}_{\vec{a}'(t+1)} -\vec{\mu}_{\vec{a}(t+1)} - 2\hspace{-1em}\sum_{q = 1}^{|\mathcal{N}_{nz}(\vec{a}(t+1))|}\hspace{-1em}a_{h_q}(t+1) C_{t, c_{h_q}} \\ 
&\hspace{0.5em}\geq \vec{\mu}_{\vec{a}'(t+1)} -\vec{\mu}_{\vec{a}(t+1)} - La_{max}\sqrt{\frac{4(L+1)\log(T\sqrt[4]{D})}{l}} \\
&\hspace{0.5em}\geq \vec{\mu}_{\vec{a}'(t+1)} -\vec{\mu}_{\vec{a}(t+1)} - \Delta_{\vec{a}(t+1)} \geq \vec{0}.
\end{align*}
The last expression above implies that (\ref{event3}) is false as long as $l$ is at least as large as the given bound.
Therefore, by setting $l = \bigg\lceil \frac{4a^2_{\max}L^2(L+1)\log(T\sqrt[4]{D})}{\Delta_{\min}^2} \bigg\rceil$, we make the probability of this event zero.
Combining our results and plugging them in (\ref{last_bound}) we obtain the following:
\begin{align*}
\E[\tilde{T}_i(T)] &\leq \bigg\lceil \frac{4a^2_{\max}L^2(L+1)\log(T\sqrt[4]{D})}{\Delta_{\min}^2} \bigg\rceil \\ &\hspace{-0.5em}+ \sum_{t=1}^\infty \sum_{c_{h_1}=l}^{t}\sum_{c_{h_2}=l}^{t} \cdots \sum_{c_{h_{|{\cal N}_{\text{nz}}(\vec{a}(t+1))|}}=l}^{t} \sum_{d_{k_1}=1}^{t}\\ &\hspace{2.5em}\sum_{d_{k_2}=1}^{t} \cdots \sum_{d_{k_{|{\cal N}_{\text{nz}}(\rev{\vec{a'(t+1)}})|}}=1}^{t} 2Lt^{-2(L+1)} \\
&\hspace{-0.5em}\leq \frac{4a^2_{\max}L^2(L+1)\log(T\sqrt[4]{D})}{\Delta_{\min}^2} + 1 + 2L\sum_{t=1}^\infty \frac{1}{t^{2}} \\ &\hspace{-0.5em}\leq \frac{4a^2_{\max}L^2(L+1)\log(T\sqrt[4]{D})}{\Delta_{\min}^2} + 1 + \frac{\pi^2}{3}L.
\end{align*}
Finally, we use this result and (\ref{t_change}) to bound the expected regret of COMO-UCB.   $\hfill \qed$

\dtl{\begin{align*}
\E[\text{Reg}(T)] &\leq \Delta_{\max} \hspace{-1em}\sum_{\vec{a} \in \mathcal{A} - \O}\hspace{-1em} \E[T_{\vec{a}}(T)] \Delta_{\max} \sum_{i = 1}^N \E[\tilde{T}_{i}(T)] \\
&\leq \Delta_{\max}\Big( \frac{4a^2_{\max}NL^2(L+1)\log(T\sqrt[4]{D})}{\Delta_{\min}^2} \\ &\hspace{9em}+ N + \frac{\pi^2}{3}NL \Big). \hfill \qed
\end{align*}}

From Theorem \ref{thm:mainregret} we conclude that the expected regret of COMO-UCB is $O(N L^3 \log T)$. Moreover, $D$ affects the regret indirectly through $\Delta_{\max}$ and $\Delta_{\min}$ and directly through the term inside the logarithm. As we will show in Section \ref{sec:app}, COMO-UCB provides significant performance improvement over naive MO-MAB algorithms when $N L^3$ is much smaller than the number of actions.

\ifodd 2 
\begin{thm}
When run with $$C_{i}(t) = \sqrt{\frac{2\Big(1+2\log \big( D|\mathcal{A}|T \big)\Big)}{m_{i}(t)}},$$ the expected regret of COMO-UCB is bounded by
\begin{align*}
\mathbb{E} [\mathrm{Reg}(T)] &\leq 4|\mathcal{A}|L\sqrt{T} \sqrt{2(1+2\log(D|\mathcal{A}|T))} + \Delta_{max}
\end{align*}
\end{thm}

\begin{proof}
We define the lower and upper confidence bounds of an the $i^{th}$ arm as $$\vec{L}_i(t) = \hat{\vec{\mu}}_i(t) - C_i(t), \>\>\> \vec{U}_i(t) = \hat{\vec{\mu}}_i(t) + C_i(t).$$
Let $$\text{UC}_i^{(d)}(t) = \bigcup_{t=1}^{m_i(t)} \Big( \hat{{\mu}}_i^{(d)} \not \in [L^{(d)}_i(t), U_i^{(d)}(t)] \Big),$$
$\mathrm{UC}_i = \cup_{d=1}^D \mathrm{UC}_i^{(d)}$ and $\mathrm{UC} = \cup_{i=1}^N \mathrm{UC}_i$ denote what we call an uncertain event, where the mean estimates of an arm falls outside the interval of upper and lower confidence bounds. Then, we can condition our regret on this event as the following.
\begin{align}\label{partition}
&\E [\text{Reg}(T)] \notag \\ &\hspace{1em}=\E[\text{Reg}(T)| \text{UC}] \P(\text{UC}) + \E[\text{Reg}(T)| \text{\text{UC}}^c] \P(\text{UC}^c) \notag \\ &\hspace{1em}\leq T \Delta_{max} \P(\text{UC}) + \E[\text{Reg}(T)| \text{\text{UC}}^c]
\end{align}
We now proceed by bounding $\P(\text{UC})$ using the following lemma.
\begin{lemma}[\cite{abbasi, russo}]
Suppose that $\forall i \in \{1,2,..,N\}$ and $\forall d \in \{1,2,...,D\}$, the difference $X_i^{(d)}(t) - \mu_i^{(d)}$ is conditionally 1-sub-Gaussian. Then, $\forall \delta > 0$, the following holds $\forall t \in \mathcal{N}$ with at least $1-\delta$ probability.
\begin{align}\label{concentration}
| \hat{\vec{\mu}}_{i}(t) - \vec{\mu}_{i} | \leq \sqrt{\frac{2}{m_i(t)}\Bigg(1+2\log \Bigg( \frac{\sqrt{1+m_i(t)}}{\delta} \Bigg) \Bigg)}.
\end{align}
\end{lemma}

By setting $\delta = \frac{1}{TD|\mathcal{A}|}$, we obtain the following.
\begin{align}\label{unconfident}
\P (\text{UC}^{(d)}_i) \leq \frac{1}{TD|\mathcal{A}|} \notag \\
\P (\text{UC}_i) \leq \frac{1}{T|\mathcal{A}|} \notag \\
\P (\text{UC}) \leq \frac{1}{T}.
\end{align}

Now, we will bound $\E[\text{Reg}(T)|\text{UC}^c]$. Let $\vec{a}(t)$ denote the arm selected at time $t$ and let $\O_{\vec{a}(t)}$ denote the subset of SPF that super dominates $\vec{a}(t)$. If $\vec{a}(t) \not \in \O$, then this set is non-empty. Next, we define $\vec{a}'(t)$ as the action that dominates $\vec{a}(t)$ the most, which is given as
\begin{numcases}{\vec{a}'(t)=}
  \argmax_{\vec{a}^\ast \in \O_{\vec{a}(t)}} 
  \bigg\{ \min_{1 \leq j \leq D} \big\{ \mu_{\vec{a^\ast}}^{(j)} - \mu_{\vec{a}(t)}^{(j)} \big\} \bigg\}, & \hspace{-1.5em}if $\vec{a}(t) \not\in \O$\nonumber \\
  \hspace{12.5em}\vec{a}(t), &\hspace{-1.9em} otherwise. \nonumber
\end{numcases}

Similarly, we define $$d_{\vec{a}(t)} = \argmin_{1 \leq d \leq D} \bigg\{ \mu_{\vec{a}(t)}^{(d)} - \mu_{\vec{a}'(t)}^{(d)} \bigg\}.$$

Let $\mathcal{T}_{\vec{a}}(T)$ denote the set of time indices when the action $\vec{a}$ is chosen until time $T$ and $N_{\vec{a}}(T) = |\mathcal{T}_{\vec{a}}(T)|$. In addition, let $B_{T} = 2\sqrt{2(1 + 2\log(D|\mathcal{A}|T))}$. Then, we have
\begin{align*}
&\E[\text{Reg}(T) | \text{UC}^c ] = \sum_{\vec{a} \in \mathcal{A}} \sum_{t \in \mathcal{T}_{\vec{a}}(T)} (\mu_{\vec{a}'}^{(d_{\vec{a}})} - \mu_{\vec{a}}^{(d_{\vec{a}})}) \\
&\hspace{1em}\leq \sum_{\vec{a} \in \mathcal{A}} \sum_{t \in \mathcal{T}_{\vec{a}}(T)} \Bigg( \sum_{i \in \mathcal{N}_{nz}(\vec{a})} (U_i^{(d_{\vec{a}})} - L_i^{(d_{\vec{a}})}) \Bigg) \\
&\hspace{1em}\leq \sum_{\vec{a} \in \mathcal{A}} \sum_{t \in \mathcal{T}_{\vec{a}}(T)} \Bigg( \sum_{i \in \mathcal{N}_{nz}(\vec{a})} \big( 2C_i^{(d_{\vec{a}})}(t) \big) \Bigg) \\
&\hspace{1em}\leq B_T \sum_{\vec{a} \in \mathcal{A}} \sum_{t \in \mathcal{T}_{\vec{a}}(T)} \Bigg( \sum_{i \in \mathcal{N}_{nz}(\vec{a})} \sqrt{\frac{1}{m_i(t)}} \>\> \Bigg) \\
&\hspace{1em} \leq B_T \sum_{\vec{a} \in \mathcal{A}} \sum_{t \in \mathcal{T}_{\vec{a}}(T)} \Bigg( \sum_{i \in \mathcal{N}_{nz}(\vec{a})} \sum_{k=0}^{N_{\vec{a}}-1} \sqrt{\frac{1}{1+k}} \Bigg).
\end{align*}
We use the following fact to bound $\E[\text{Reg}(T) | \text{UC}^c ]$ further. $$\sum_{k=0}^{N_{\vec{a}}(T)-1} \sqrt{\frac{1}{1+k}} \leq \int_{x=0}^{N_{\vec{a}}(t)}\frac{1}{\sqrt{x}}dx = 2\sqrt{N_{\vec{a}}(T)}$$
Then,
\begin{align}\label{confident}
&\E[\text{Reg}(T) | \text{UC}^c ] \leq 2 B_T \sum_{\vec{a} \in \mathcal{A}} \Bigg( \sum_{i \in \mathcal{N}_{nz}(\vec{a})} \sqrt{N_{\vec{a}}(T)} \Bigg) \notag \\
&\hspace{1em}\leq 4 \sqrt{2(1+2\log(D|\mathcal{A}|T))} \sum_{\vec{a} \in \mathcal{A}} L\sqrt{N_{\vec{a}}(T)} \notag \\
&\hspace{1em}\leq 4|\mathcal{A}|L\sqrt{T} \sqrt{2(1+2\log(D|\mathcal{A}|T))}
\end{align}
Then, plugging (\ref{confident}) and (\ref{unconfident}) into (\ref{partition}) finalizes the proof.
\begin{align*}
&\E [\text{Reg}(T)] \\ &\hspace{1em}=\E[\text{Reg}(T)| \text{UC}] \P(\text{UC}) + \E[\text{Reg}(T)| \text{\text{UC}}^c] \P(\text{UC}^c) \\ &\hspace{1em}\leq T \Delta_{max} \P(\text{UC}) + \E[\text{Reg}(T)| \text{\text{UC}}^c] \\
&\hspace{1em}\leq 4|\mathcal{A}|L\sqrt{T} \sqrt{2(1+2\log(D|\mathcal{A}|T))} + \Delta_{max}
\end{align*}

\end{proof}
\fi 

\section{Applications of COMO-MAB} \label{sec:app}

\subsection{Multi-User Communication} \label{sec:commproblem}

In the past MAB was used to model multi-user \cite{liu2010distributed,anandkumar2011distributed} opportunistic spectrum access, and learning of optimal transmission parameters in wireless communications \cite{gulati2014learning}.
An important aspect of multi-user communication that is overlooked in prior works is the multidimensional nature of the performance metrics of interest. For instance, applications such as real-time streaming are concerned with the metrics of end-to-end delay, achieved throughput, as well as the delivery ratio in order to achieve a good quality of service. In contrast, sensing, monitoring and control applications are more concerned with regularity (or periodicity) and freshness of its updates in order to assure stable and efficient tracking and control of its network. Therefore, it is of-interest to develop mechanisms that can tradeoff between multiple metrics that govern the performance of the networks. \dtl{In prior works (e.g. \cite{hou2009theory,li2011unified}), this problem is predominantly approached either by scalarizing the multi-dimensional performance metric by using scalarization coefficients or by using some metrics as constraints while optimizing a single dimensional objective. However, the optimal solution in such approaches depends on the values of the scalarization coefficients or the bounds on the constraints.}

Consider the service of $M$ users over $Q$ channels ($Q \geq M$), which takes place in a sequence of discrete time steps indexed by $t \in \{1,2,\ldots\}$. We define ${\cal M} := [M]$ to be the set of users and ${\cal Q} := [Q]$ to be the set of channels. The channel gain for user $i$ and channel $j$, denoted by $h^2_{i,j}$, is exponentially distributed with parameter $\lambda_{i,j}$. \rev{This distribution is unknown.} We assume that the channel gain is fixed during a time step and user $i$ can choose its transmission rate $R_{\text{tx}}$ over channel $j$ from $H$ different transmission rates at each time step. We define ${\cal H}_{i,j} := \{R_{i,j,1}, \ldots R_{i,j,H} \}$ to be the set of transmission rates that user $i$ can use over channel $i$, where $R_{i,j,k} < R_{i,j,k+1}$ for all $k \in \{1,\ldots,H-1\}$. Therefore, each arm corresponds to a particular user-channel-transmission rate assignment indexed by $(i,j,k)$ and we have $N = MQH$.

If user $i$ transmits at rate $R_{\text{tx}}$ over channel $j$, two rewards are produced: throughput and reliability. Here, throughput measures the successful average rate of communication between the transceivers, while reliability concerns the success rate of transmissions over time. These two metrics/objectives are typically in conflict in that achieving high reliability typically requires a low rate of communication. We note that this choice of multi-dimensional metrics is only one of many that can be incorporated into our general setting. For example, we can use energy consumption, service regularity, information freshness as other metrics of interest. 
%

\rev{There is a base station (learner) which acts as a central controller. The base station takes an action at each time step to decide which users will be assigned to which channels and which transmission rates will be used in that time step}. Then, each user will make a transmission in their assigned channels. At the end of the time step, the base station receives the success/failure event and the achieved throughput of the transmission.  These two parameters constitute the two dimensional performance of the action. 
It is assumed that the users are within the interference range of each other and cannot simultaneously use the same channel, and hence, the feasible channel allocations have a one-to-one matching of users to channels. 

 Each allocation is represented by $\bs{a} := [a_{i,j,k}]$, where $a_{i,j,k}$ is $1$ if user $i$ is assigned to channel $j$ and uses transmission rate $R_{i,j,k}$, and $0$ otherwise. Based on this, the set of actions is defined as
${\cal A} := \{ \bs{a}:  \sum_{j=1}^Q \sum_{k \in {\cal H}} a_{i,j,k} = 1, \forall i \in {\cal M} \text{, } 
\sum_{i=1}^M \sum_{k \in {\cal H}} a_{i,j,k}   \leq 1, \forall j \in {\cal Q} \}$.

The $2$-dimensional random reward vector of user $i$ in channel $j$ when it transmits at rate $R_{i,j,k}$ at time step $t$ is denoted by $\bs{X}_{i,j,k}(t) := [X^{(1)}_{i,j,k}(t), X^{(2)}_{i,j,k}(t)]$, where $X^{(1)}_{i,j,k}(t) \in \{0,1\}$ denotes the success (1) or failure (0) event, and $0 \leq X^{(2)}_{i,j,k}(t) \leq 1$ denotes the normalized achieved throughput of the transmission. The mean vector of arm $(i,j,k)$ is denoted by $\bs{\mu}_{(i,j,k)}$. 
Based on the definitions given above, we have $\mu^1_{(i,j,k)} = 1 - p_{\text{out}}(i,j,k)$, where
$p_{\text{out}}(i,j,k) := \Pr (\log (1 + h^2_{i,j} \text{SNR} ) < R_{i,j,k} )$
denotes the outage probability, and $\mu^2_{(i,j,k)} = R_{i,j,k} (1 - p_{\text{out}}(i,j,k)) / R_{i,j,H}$ denotes the normalized average throughput, where $\text{SNR}$ is the signal to noise ratio.

Based on this, the random reward and mean reward vectors of action $\bs{a}$ at time step $t$ are given by 
$\bs{R}_{\bs{a}}(t) = \sum_{i=1}^{M} \sum_{j=1}^Q \sum_{k=1}^{H} a_{i,j,k} \bs{X}_{i,j,k}(t)$ and
$\bs{\mu}_{\bs{a}} = \sum_{i=1}^{M} \sum_{j=1}^Q \sum_{k=1}^{H} a_{i,j,k} \bs{\mu}_{i,j,k}$,
respectively. 
The goal of the base station is to simultaneously maximize the long term reward in both objectives.

The next corollary bounds the expected regret of COMO-UCB for the multi-user multi-objective communication problem. 
\begin{corollary}
For the multi-user multi-objective communication problem, the expected regret of COMO-UCB run with the same $C_i(t)$ value in Theorem \ref{thm:mainregret} is bounded by
$\mathbb{E} [\mathrm{Reg}(T)] \leq \Delta_{\max}\Big( \frac{4 M^3 Q H (M+1)\log(T\sqrt[4]{2})}{\Delta_{\min}^2} + MQH + \frac{\pi^2}{3} M^2 Q H \Big)$.
\end{corollary}
The corollary above shows that the regret of COMO-UCB is a polynomial function of $M$, $Q$ and $H$. As an alternative to COMO-UCB, one could have used the multi-objective learning algorithm developed in \citet{drugan13} by treating each action as a separate arm. The regret of this algorithm grows linearly in the number of actions that are not in the Pareto front, and it requires to hold and update the sample mean reward estimates for all the actions.  Since the cardinality of the action space in this case is $|{\cal A}| = H^M Q (Q-1) \ldots (Q-M+1)$, this algorithm is inefficient both in terms of the regret and the memory complexity for the multi-user multi-objective communication problem that we consider in this paper.

\dtl{$a_{\max} = 1$, $N = MQH$, $L = M$ since each of the $M$ users select only one configuration, $D=2$.}

\begin{figure*}
	\begin{minipage}{\textwidth}
		\begin{adjustbox}{max width=\textwidth}
			\begin{minipage}{\columnwidth}
				\centering
				\includegraphics[width=0.98\columnwidth]{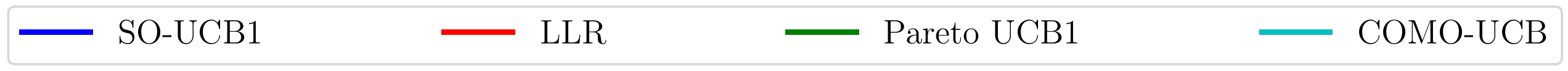}
			\end{minipage}%
		\end{adjustbox}
	\end{minipage}
	\begin{minipage}{\textwidth}
		\begin{adjustbox}{max width=\textwidth}
			\begin{minipage}{0.66\columnwidth}
				\centering
				\includegraphics[width=\columnwidth]{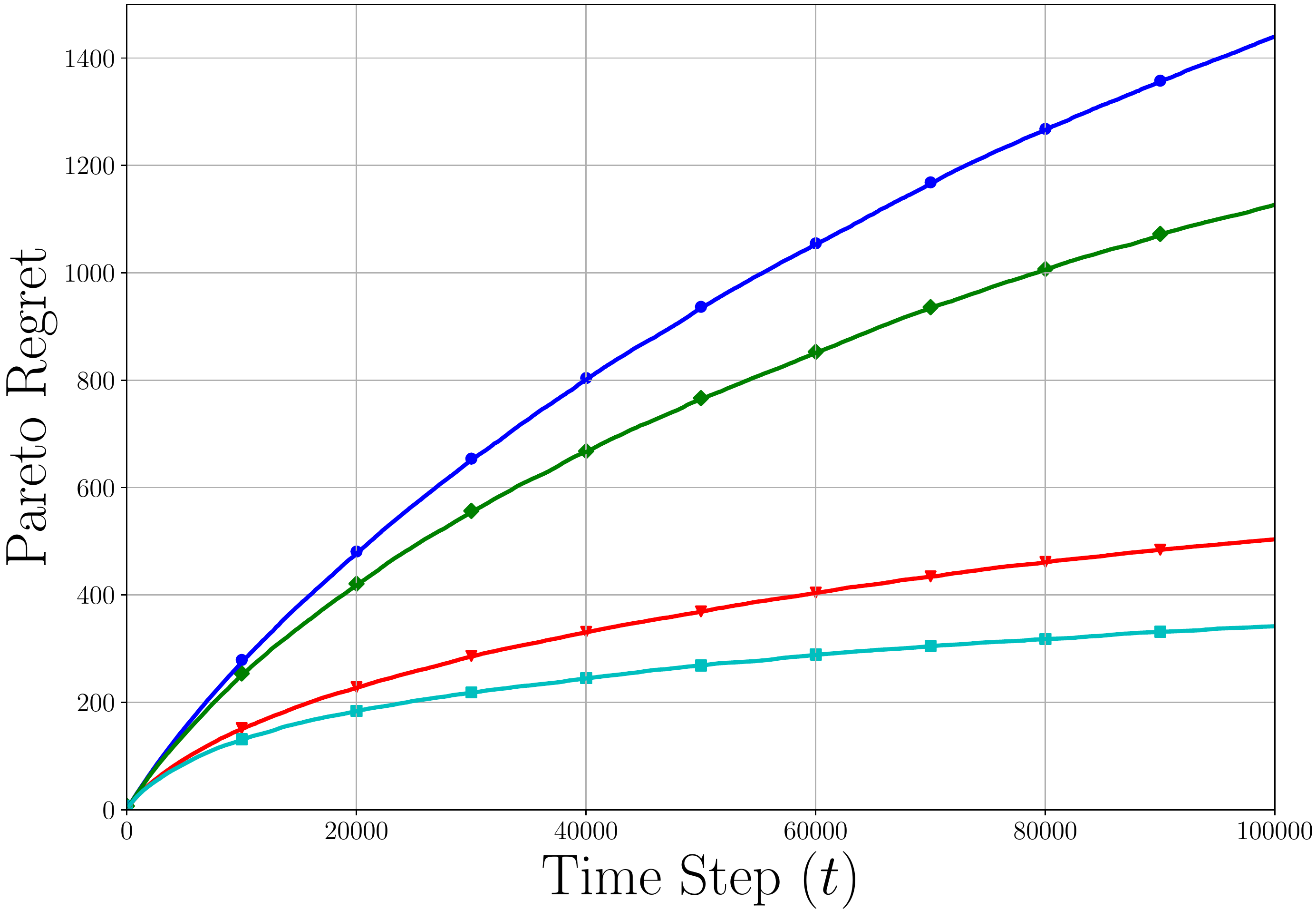}
				\caption{Pareto Regret Comparison}\label{regret}
			\end{minipage}%
			\begin{minipage}{0.66\columnwidth}
				\centering
				\includegraphics[width=\columnwidth]{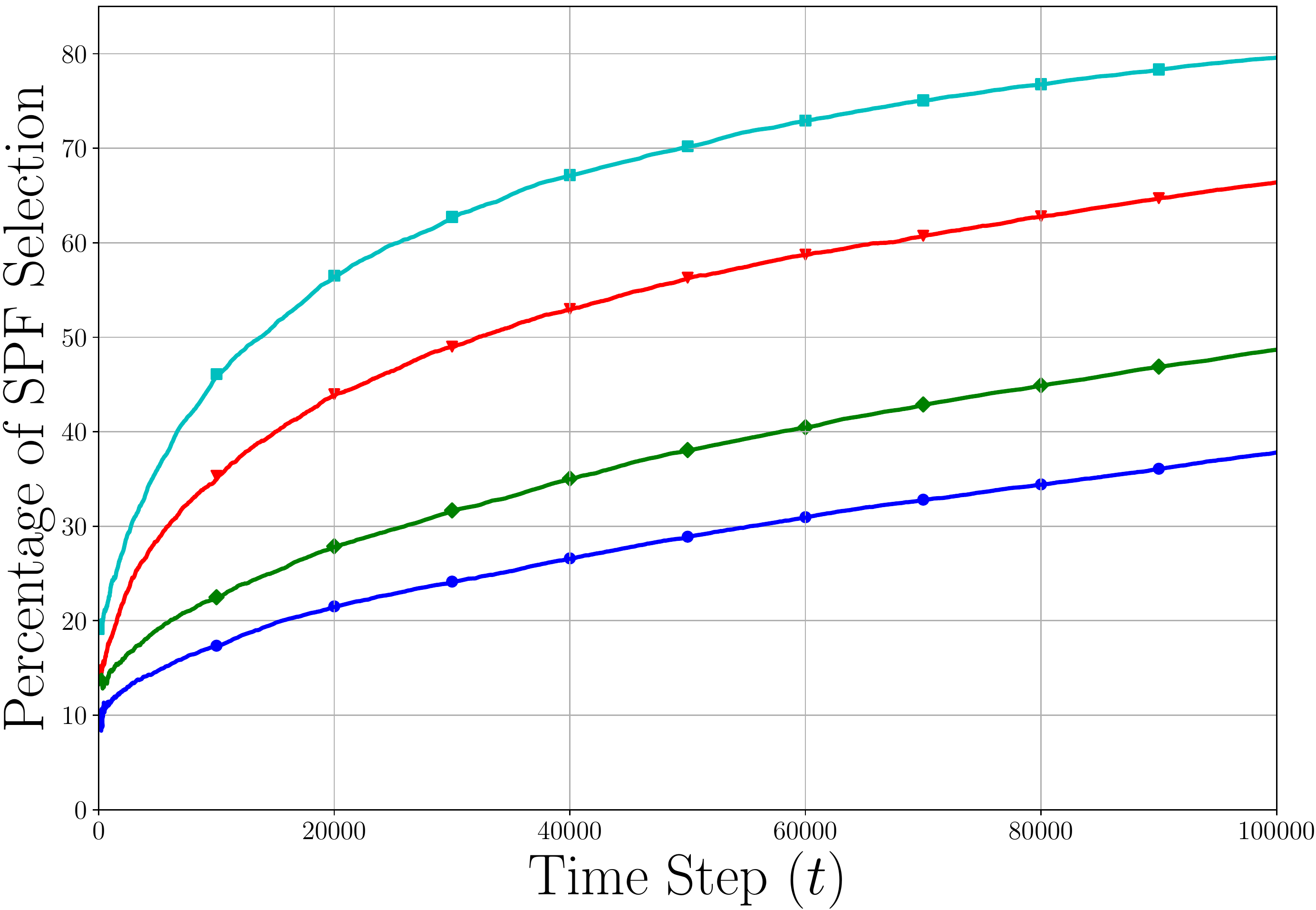}
				\caption{Percentage of SPF Selections}\label{percent}
			\end{minipage}
			\begin{minipage}{0.62\columnwidth}
				\centering
				\includegraphics[width=\columnwidth]{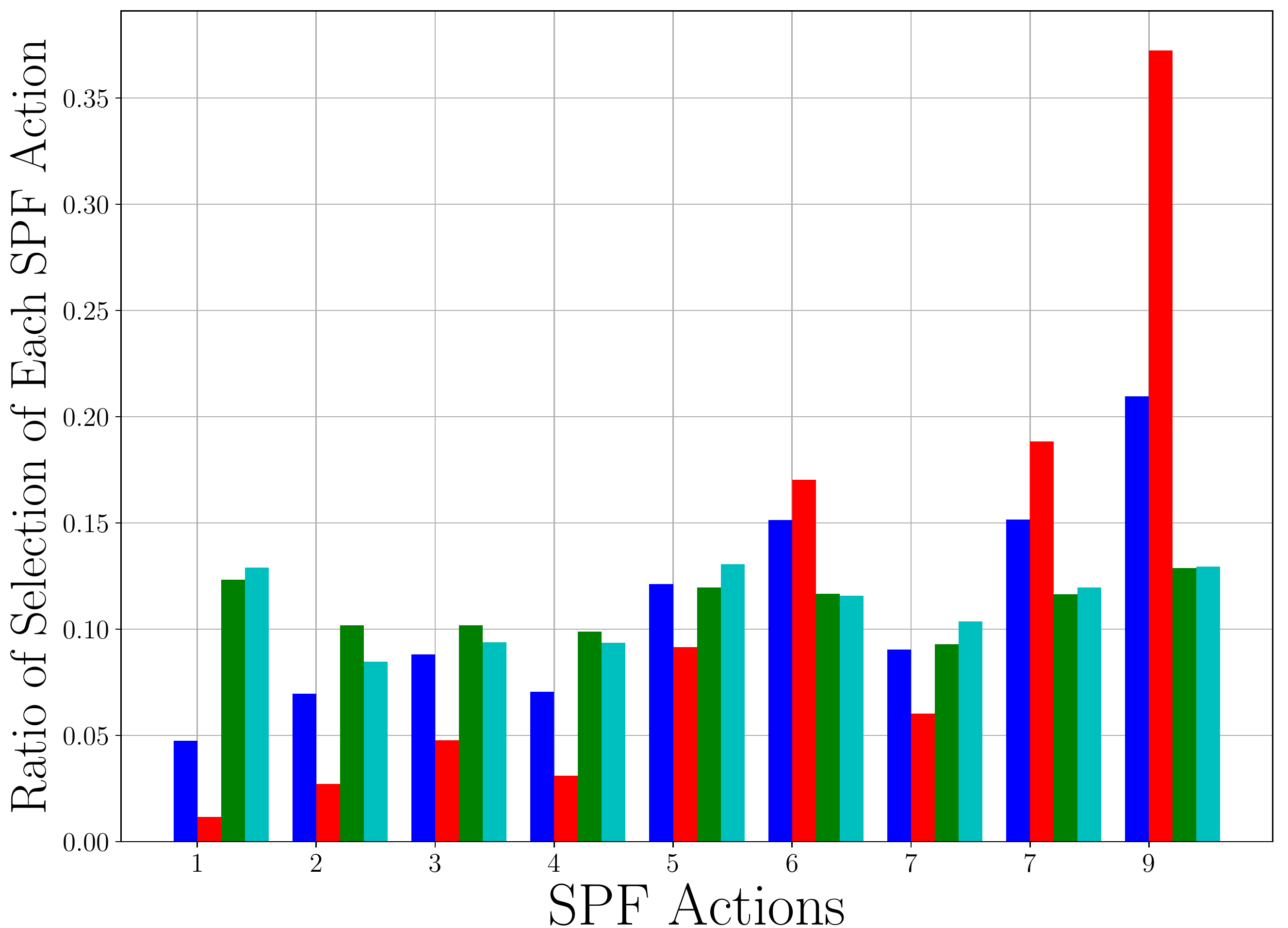}
				\caption{Fairness Comparison}\label{fair}
			\end{minipage}
		\end{adjustbox}
	\end{minipage}
\end{figure*}

\subsection{Recommender System}
Recommender systems involve optimization of multiple metrics like novelty and diversity \cite{rec1, rec2} in addition to average rating. Below, we describe how a recommender system with average rating and diversity metrics can be modeled using COMO-MAB.

Consider a recommender system recommending $K$ out of $N$ items to $M$ similar users that arrive at each time step, which have the same observable context $x_o$.\footnote{In general, a different instance of COMO-UCB can be run for each set of similar users.} Let $\bs{U}_j(t) = [U_{j1}(t), \ldots, U_{jN}(t)]$ denote the rating vector of user $j$, where $U_{ji}(t) =1$ if user $j$ likes item $i$ and $0$ otherwise. The distribution of $\bs{U}_j(t)$ is given as $p_{j}(x_j,x_o)$ where $x_j$ is the hidden context of user $j$, which is drawn from a fixed distribution defined over a context set ${\cal X}_{x_o}$ independently from the other users. Neither $x_j$ nor $p_{j}(x_j,x_o)$ is known by the recommender system. 

The recommendations are represented by $\bs{a}$ where $a_{i} =1$ if item $i$ is recommended and $0$ otherwise. Thus, the set of actions is given as
${\cal A} = \{ \bs{a}: a_{i} \in \{0,1\}, \forall i \in {\cal N} \text{ and } \sum_{i=1}^{N} a_{i} = K \}$. The random reward vector is $2$-dimensional $\bs{X}_{i}(t) = [X^{(1)}_{i}(t),X^{(2)}_{i}(t)]$. Here, $X^{(1)}_{i}(t)$ is the average number of users that liked item $i$ and $X^{(2)}_{i}(t)$ denotes the cosine diversity of users that liked item $i$, which is given as $\sum_{j \neq l} c_{j,l} / (M (M-1))$, where $c_{j,l} = 1 - \tilde{\bs{U}}_j(t) \tilde{\bs{U}}^{T}_l(t) / (|| \tilde{\bs{U}}_j(t) || || \tilde{\bs{U}}_l(t) ||) $ and $\tilde{\bs{U}}_j(t)$ is the $1$ by $K$ vector that consists of entries of $\bs{U}_j(t)$ that correspond to the recommended items. As an alternative, $X^{(2)}_{i}(t)$ can also represent the sample variance of the ratings of the users for item $i$. 
The next corollary bounds the expected regret of COMO-UCB for the above recommendation problem. 
\begin{corollary}
	For the recommender system, the expected regret of COMO-UCB run with the same $C_i(t)$ value in Theorem \ref{thm:mainregret} is bounded by
	$\mathbb{E} [\mathrm{Reg}(T)] \leq \Delta_{\max}\Big( \frac{4 N K^2 (K+1) \log(T\sqrt[4]{2})}{\Delta_{\min}^2} + N + \frac{\pi^2}{3} N K \Big)$.
\end{corollary}
Note that a learning algorithm that does not exploit the combinatorial nature of this problem will incur regret proportional to $\binom{N}{K} \log T$.

\dtl{$a_{\max}=1$, $L=K$, $|{\cal N}| = N$, $D=2$.}

\subsection{Network Routing} Packet routing in a communication network commonly involves multiple paths that can be modeled as combinatorial selections of edges of a given graph. Adaptive packet routing can improve the performance by avoiding congested and faulty links. In many networking problems, it is desirable to minimize energy consumption as well as the delay due to the energy constraints of Internet of Things devices and sensor nodes.

Given a source destination pair $(s,d)$, we can formulate routing of the flow from node $s$ to node $d$ as a COMO-MAB problem. Let $-X_{(l,k)}^{(1)}(t)$ and $-X_{(l,k)}^{(2)}(t)$ denote the random delay and energy consumption incurred on the edge between nodes $l$ and $k$, respectively.\footnote{These can also be normalized to lie in the unit interval.} The action set is the set of paths connecting $s$ to $d$, and each action is a path from $s$ to $d$. Thus, the learner observes all the rewards in edges $(l,k) \in \vec{a}(t)$, and collects reward $X^{(j)}(t) = \sum_{(l,k) \in \vec{a}(t)} X_{(l,k)}^{(j)}(t)\text{ for }j=1,2.$
The next corollary bounds the expected regret of COMO-UCB for this problem. 
\begin{corollary}
	For network routing, the expected regret of COMO-UCB run with the same $C_i(t)$ value in Theorem \ref{thm:mainregret} is bounded by
	$\mathbb{E} [\mathrm{Reg}(T)] \leq \Delta_{\max}\Big( \frac{4 N L^2 (L+1) \log(T\sqrt[4]{2})}{\Delta_{\min}^2} + N + \frac{\pi^2}{3} N L \Big)$,
where $L$ is the length of the longest acyclic path from $s$ to $d$.
\end{corollary}
Similar to the previous section, a learning algorithm that treats each path as an arm incurs regret proportional to the number of paths from $s$ to $d$.

\vspace{-0.1in}
\section{Numerical Results}\label{sec:numerical}
We consider the multi-user communication problem given in Section \ref{sec:commproblem}, where $M=2$, $Q=4$ and $H=3$. In this case, the actions are $24$-dimensional (represented by a $2$ by $4$ by $3$ matrix), and the total number of actions is equal to $12 \times 9 = 108$. $\lambda_{i,j}s$ are selected randomly from the interval $[0.05, 0.2]$, and are set as: 
\begin{align*}
[\lambda_{i,j} ] =
\begin{bmatrix}
0.14 & 0.14 & 0.16 & 0.05 \\
0.05 & 0.11 & 0.13 & 0.07 \\
\end{bmatrix}
\end{align*}
In addition, SNR is taken to be $1$ and $R_{i,j,1} = R_{i,j}/4$, $R_{i,j,2} = R_{i,j}/2$ and $R_{i,j,3} = R_{i,j}$, where $R_{i,j} := \text{ProductLog}[15\lambda_{i,j}]$. In this setup, SPF $=$ Pareto front and $9$ out of $108$ actions are in SPF. The time horizon $T$ is taken as $10^5$ and all reported results are averaged over $5$ runs.

Figure \ref{regret} shows the regrets of COMO-UCB and the competitor algorithms Pareto UCB1 from \citet{drugan13}, Learning with Linear Rewards (LLR) from \citet{gai} and Single Objective UCB1 (SO-UCB1), which is the same as UCB1 in \citet{ucb}, as a function of $t$. Pareto UCB1 treats each action as a separate arm, and at each time step only updates the parameters of the selected action. Moreover, it also takes as input the size of the Pareto front, which is not required by COMO-UCB. LLR is a combinatorial algorithm that works with a scalar reward. Instead of calculating the Pareto front, it aims at selecting the action that maximizes the reward in the first objective. On the other hand, SO-UCB1 treats each action as a separate arm and tries to maximize the reward in the first objective. It can be seen from Figure \ref{regret} that \dtl{the regret of COMO-UCB is nearly one fourth of the regret of SO-UCB, one third of the regret of Pareto UCB1 and 1.5 times lower than the regret of LLR} the regret incurred by COMO-UCB grows significantly slower than the other algorithms. This is due to the fact that COMO-UCB finds the SPF much faster than other algorithms by exploiting the dependence between the actions and by keeping track of the rewards in both objectives.

Figure \ref{percent} reports the fraction of times an action from the Pareto front is selected as a function of $t$. 
At the end of $10^5$ time steps, COMO-UCB selects an action from the Pareto front $79\%$ of the time, while SO-UCB1 selects an action from the Pareto front only $37\%$ of the time, Pareto UCB1 selects $48\%$ of the time and LLR selects $67\%$ of the time.

We also compare the algorithms in terms of their fairness. In Figure \ref{fair}, we used a bar chart to represent the fraction of times that each one of the $9$ actions in the SPF is selected during the time steps in which an action from the SPF is selected by the algorithms. We say that an algorithm is fair if these fractions are close for all $9$ arms. 
We observe that fairness of LLR is much worse than fairness of COMO-UCB, even though LLR is the closest competitor to COMO-UCB in terms of the Pareto regret. We conclude that COMO-UCB and Pareto UCB1, which select actions from the SPF uniformly at random are fair. However, SO-UCB1 and LLR selects the $9$th action in the SPF significantly more than other actions in the SPF. This is expected, since these algorithms aim to maximize only the reward in the first objective.
\vspace{-0.1in}
\section{Conclusion}\label{sec:conc}
We proposed a new MAB model, called COMO-MAB, that combines combinatorial bandits with multi-objective online learning, and designed a learning algorithm that achieves $O(N L^3 \log T)$ Pareto regret. We showed that COMO-MAB can be used to model various multi-objective problems in multi-user communication, recommender systems and network routing. Then, we validated the effectiveness of the proposed algorithm through simulations in a multi-user communication problem. 

\section*{Acknowledgement}
The work of C. Tekin was supported by the Scientific and Technological Research
Council of Turkey (TUBITAK) under 3501 Program Grant No. 116E229.

\vspace{-0.2in}

\bibliography{references}
\bibliographystyle{icml2018}

\end{document}